\documentclass{article}
\usepackage{fullpage}
\usepackage[title]{appendix}
\usepackage[utf8]{inputenc} % allow utf-8 input
\usepackage[T1]{fontenc}    % use 8-bit T1 fonts
\usepackage{hyperref}       % hyperlinks
\usepackage{url}            % simple URL typesetting
\usepackage{geometry}
\usepackage[table]{xcolor}

\usepackage{algorithm}
\usepackage{algpseudocode}
\usepackage{diagbox}
\usepackage{multirow}
\usepackage{booktabs}
\usepackage{amsfonts,amssymb,amsthm,amsmath,mathtools}
\usepackage{nicefrac}
\usepackage{microtype}
\usepackage{hhline}
\usepackage{subfigure}
\usepackage{indentfirst}
\usepackage{enumitem}
\usepackage{graphicx}
\usepackage{cleveref}
\usepackage{colortbl}
\usepackage{makecell}
\usepackage{wrapfig}
\usepackage{adjustbox}
\usepackage{array}
\usepackage{caption}
\usepackage{subcaption}
\usepackage{arydshln}
\usepackage{utfsym,appendix}

\newtheorem{definition}{Definition}
\newtheorem{theorem}{Theorem}
\newtheorem{remark}{Remark}
\newtheorem{lemma}{Lemma}

\newtheorem{corollary}{Corollary}
\theoremstyle{plain}
\theoremstyle{definition}

\usepackage{amsmath,amsfonts,bm}

\def\eqref#1{equation~\ref{#1}}
\def\1{\bm{1}}
\def\0{\bm{0}}

\DeclareMathAlphabet{\mathsfit}{\encodingdefault}{\sfdefault}{m}{sl}
\SetMathAlphabet{\mathsfit}{bold}{\encodingdefault}{\sfdefault}{bx}{n}

\usepackage{inputenc}
\usepackage{fontenc}
\usepackage{indentfirst}
\usepackage{cite}
\usepackage{hyperref}
\usepackage{amssymb}
\usepackage{float}
\usepackage{amsmath}
\usepackage{graphicx}
\usepackage{amsthm}

\newtheorem{proposition}{Proposition}

\newcommand{\citet}{\cite}
\newcommand{\citep}{\cite}

\title{The Maximum von Neumann Entropy Principle:\\ Theory and Applications in Machine Learning 
}

\date{}

\author{
Youqi Wu\thanks{Department of Computer Science and Engineering, The Chinese University of Hong Kong, yqwu24@cse.cuhk.edu.hk} ,
Farzan Farnia\thanks{Department of Computer Science and Engineering, The Chinese University of Hong Kong, farnia@cse.cuhk.edu.hk}
}

\begin{document}
\maketitle

\begin{abstract}
    Von Neumann entropy (VNE) is a fundamental quantity in quantum information theory and has recently been adopted in machine learning as a spectral measure of diversity for kernel matrices and kernel covariance operators. While maximizing VNE under constraints is well known in quantum settings, a principled analogue of the classical maximum entropy framework, particularly its decision theoretic and game theoretic interpretation, has not been explicitly developed for VNE in data driven contexts. In this paper, we extend the minimax formulation of the maximum entropy principle due to Grünwald and Dawid to the setting of von Neumann entropy, providing a game-theoretic justification for VNE maximization over density matrices and trace-normalized positive semidefinite operators. This perspective yields a robust interpretation of maximum VNE solutions under partial information and clarifies their role as least committed inferences in spectral domains. We then illustrate how the resulting Maximum VNE principle applies to modern machine learning problems by considering two representative applications, selecting a kernel representation from multiple normalized embeddings via kernel-based VNE maximization, and completing kernel matrices from partially observed entries. These examples demonstrate how the proposed framework offers a unifying information-theoretic foundation for VNE-based methods in kernel learning.
\end{abstract}

\section{Introduction}
Entropy is a fundamental quantity in information theory, quantifying the uncertainty of a random variable and playing a central role in data compression, communication, and statistical inference. In the classical setting, Shannon entropy \cite{Shannon1948} provides a mathematical characterization of uncertainty for a probability distribution. In quantum information theory \cite{vonNeumann1955,NielsenChuang2000}, this notion is extended to quantum states, which are represented by density matrices. A density matrix $\rho$ is a positive semidefinite (PSD) operator with unit trace, and its associated entropy, known as the \emph{von Neumann entropy (VNE)}, is defined as the Shannon entropy of its eigenvalues:
\begin{equation}
S(\rho) \;=\; -\mathrm{Tr}(\rho \log \rho) \;=\; -\sum_i \lambda_i(\rho)\log \lambda_i(\rho),
\end{equation}
where the eigenvalues $\{\lambda_i(\rho)\}$ form a probability distribution.

Recently, VNE has been adopted in machine learning as a tool for measuring diversity and effective dimensionality of data representations \cite{bach2022informationtheorykernelmethods,FriedmanDieng2022Vendi}. In particular, kernel matrices and kernel covariance operators, when normalized to have unit trace, naturally define density-matrix-like objects whose spectra encode global similarity structure \cite{jalali2023information,ospanov2024towards}. This observation has led to the use of VNE-based quantities, such as the Vendi score \cite{FriedmanDieng2022Vendi}, as spectral diversity measures for datasets and generative models. However, in several machine learning applications, the underlying density matrix is not fully observed, for example due to missing entries of the target kernel matrix, noisy measurements, or kernel covariance operators estimated from limited data. This motivates a fundamental question: how should VNE-based analysis be carried out when the density matrix is only partially specified?

In classical information theory, entropy-based statistical inference under partial knowledge can be performed by the \emph{maximum entropy principle}. Introduced by Jaynes in \cite{Jaynes1957}, this principle advocates selecting, among all probability distributions consistent with the given constraints, the distribution with the maximum Shannon entropy value. A canonical example is that the Gaussian distribution maximizes (differential) entropy among all distributions with fixed mean and covariance matrix \cite{CoverThomas2006}. Beyond its variational formulation, the maximum entropy principle admits an operational interpretation as demonstrated by Gr{\"u}nwald and Dawid in \cite{GrunwaldDawid2004}, showing that the maximum entropy solution arises as a minimax-optimal strategy in a game-theoretic setting where a decision maker competes against an adversary over an ambiguity set of distributions.

In this paper, we extend this framework to von Neumann entropy, with a focus on machine learning settings involving partial knowledge of a density matrix. Given an ambiguity set $\mathcal{C}$ of density matrices consistent with the available information, we consider the \emph{maximum VNE principle}, i.e. to choose and base the inference on the density matrix $\rho^*$ with the largest VNE:
\begin{equation}
\rho^\star \;\in\; \arg\max_{\rho \in \mathcal{C}} S(\rho).
\end{equation}
Our main contribution is to formally show that the game-theoretic interpretation of the maximum entropy principle developed by Gr{\"u}nwald and Dawid in \cite{GrunwaldDawid2004} extends naturally to the matrix-valued setting. This yields a principled minimax justification for maximizing VNE over the set of density matrices, directly paralleling the classical case for probability distributions.

We further show that this minimax framework applies beyond VNE to other spectral entropy functionals of interest. In particular, we consider matrix-based R\'enyi entropies, including the quadratic R\'enyi entropy of order two and the R\'enyi entropy of order infinity, which is governed by the maximum eigenvalue \cite{Renyi1961,SanchezGiraldoRaoPrincipe2015}. We also extend the framework to conditional settings, drawing on minimax formulations of supervised learning \cite{FarniaTse2016Minimax}, thereby enabling applications in which kernel objects arise jointly from inputs and outputs.

Finally, we illustrate the proposed maximum VNE principle through two concrete applications in machine learning. The first concerns the selection of a kernel representation from multiple normalized embeddings. By considering convex combinations of the corresponding kernel matrices and applying the Max-VNE principle over the mixture weights, we obtain a principled method for selecting a least-committal kernel mixture that promotes spectral diversity. The second application addresses kernel matrix completion from partial observations, where we select, among all feasible positive semidefinite completions consistent with the observed entries, the one that maximizes VNE. We complement the theoretical developments with numerical experiments on standard image datasets, demonstrating the practical relevance of the maximum VNE principle in data-driven settings.

\section{Related Work}
The maximum entropy principle originates from Shannon’s entropy \cite{Shannon1948} and Jaynes’ proposal of selecting the least-committal model under a group of constraints \cite{Jaynes1957}. In more recent literature, closely related principles replace entropy maximization by minimizing dependence subject to partial constraints. Notable examples include the \emph{minimum mutual information} principle, which generalizes maximum entropy to discriminative learning by minimizing $I(X;Y)$ under marginal and feature constraints \cite{GlobersonTishby2004MinMI}, and the \emph{minimum maximal correlation} principle, which selects joint distributions minimizing the Hirschfeld--Gebelein--R\'enyi (HGR) correlation under low-order constraints \cite{FarniaRazaviyaynKannanTse2015MinHGR,RazaviyaynFarniaTse2015DiscreteRenyi}. These dependence-minimization principles admit minimax interpretations and lead to robust learning rules when only partial information is available. Complementarily, the minimax and game-theoretic foundation of maximum entropy itself is formalized by Gr{\"u}nwald and Dawid \cite{GrunwaldDawid2004}. These works motivate extending entropy- and dependence-based minimax principles from probability distributions to matrix-valued entropy objectives.

In quantum information theory, uncertainty is described by density matrices, which are positive semidefinite, unit-trace operators, and von Neumann entropy extends Shannon entropy to this operator-valued setting via the eigenvalue distribution \cite{vonNeumann1955,NielsenChuang2000}. Maximum entropy reasoning for quantum systems is well established, particularly in quantum statistical mechanics and inference, where it is used to select a density matrix consistent with partial information. These works motivate studying maximum entropy principles directly at the level of density matrices, but they are typically grounded in physical constraints rather than data-driven ambiguity sets.

Von Neumann entropy and related spectral entropy functionals have recently gained prominence in machine learning as measures of diversity and effective dimensionality for kernel matrices and kernel covariance operators. The Vendi score defines diversity as the exponential of the von Neumann entropy of a normalized similarity matrix and has been used as an interpretable diversity metric for datasets and generative models \cite{FriedmanDieng2022Vendi}, with related similarity-based diversity families developed subsequently \cite{pasarkar2024cousins}. In parallel, kernel-based matrix information measures estimate entropy-like quantities directly from normalized Gram matrices, including matrix-based R\'enyi-type entropies and operator-level information-theoretic viewpoints \cite{SanchezGiraldoRaoPrincipe2015,bach2022informationtheorykernelmethods}. These ideas have been applied to information-theoretic evaluation of generative models and multimodal learning \cite{jalali2023information,zhang2024interpretable,ospanov2024towards}. Our work differs from these approaches by providing a game-theoretic foundation for a maximum von Neumann entropy principle and by applying it to kernel-induced density matrices under partial observation.

\section{Preliminaries}
\subsection{Entropy and the Maximum Entropy Principle}
\label{subsec:prelim:classical}
Let $\mathcal{X}$ be a finite alphabet and let $\Delta(\mathcal{X})$ denote the
probability simplex over $\mathcal{X}$. For a distribution
$P\in\Delta(\mathcal{X})$, the Shannon entropy is defined as
\begin{equation}
H(P)\;:=\;-\sum_{x\in\mathcal{X}} P(x)\log P(x).
\label{eq:shannon-entropy}
\end{equation}
Given a feasible family $\mathcal{P}\subseteq\Delta(\mathcal{X})$ specified by
constraints (for example, moment or marginal constraints), the classical
maximum entropy formulation \cite{Jaynes1957} finds and acts based on the entropy-maximizing distribution $P^*$:
\begin{equation}
P^\star \;\in\; \arg\max_{P\in\mathcal{P}} H(P).
\label{eq:maxent-classical}
\end{equation}
%Throughout the paper, $\log$ denotes the natural logarithm and all entropy quantities are measured in nats.

\subsection{Density Matrices and Matrix Entropies}
\label{subsec:prelim:matrix}
Let $\mathbb{S}^n_+$ denote the cone of $n\times n$ real symmetric positive
semidefinite matrices. A \emph{density matrix} is a matrix
$\rho\in\mathbb{S}^n_+$ satisfying
\begin{equation}
\rho \succeq 0,
\qquad
\mathrm{Tr}\bigl(\rho\bigr)=1.
\label{eq:density-matrix}
\end{equation}
Let $\lambda(\rho)=\bigl(\lambda_1(\rho),\ldots,\lambda_n(\rho)\bigr)$ denote the
eigenvalues of $\rho$ (counted with multiplicity). Since
$\rho\succeq 0$ and $\mathrm{Tr}\bigl(\rho\bigr)=\sum_i \lambda_i(\rho)=1$, we have
$\lambda(\rho)\in\Delta_n$, where $\Delta_n \;:=\; \bigl\{p\in\mathbb{R}^n_+ : \sum_{i=1}^n p_i = 1 \bigr\}$. The \emph{von Neumann entropy} (VNE) of $\rho$ is defined as
\begin{equation}
S(\rho)\;:=\;-\mathrm{Tr}\!\bigl(\rho \log \rho\bigr)
\;=\;-\sum_{i=1}^n \lambda_i(\rho)\log \lambda_i(\rho),
\label{eq:vne}
\end{equation}
where $\log \rho$ is defined via functional calculus on the eigenvalues.

More generally, for $\alpha>0$ with $\alpha\neq 1$, the order-$\alpha$ (quantum)
R\'enyi entropy is defined by
\begin{equation}
S_\alpha(\rho)\;:=\;\frac{1}{1-\alpha}\log \mathrm{Tr}\!\bigl(\rho^\alpha\bigr)
\;=\;\frac{1}{1-\alpha}\log \sum_{i=1}^n \lambda_i(\rho)^\alpha.
\label{eq:renyi}
\end{equation}
A special case will be used in this paper. The quadratic (order-$2$) R\'enyi
entropy is
\begin{equation}
S_2(\rho)\;=\;-\log \mathrm{Tr}\bigl(\rho^2\bigr),
\label{eq:renyi-2}
\end{equation}
%and the order-$\infty$ limit equals \begin{equation} S_\infty(\rho):=\lim_{\alpha\to\infty} S_\alpha(\rho)=-\log \|\rho\|_{\mathrm{op}}=-\log\Bigl(\max_i \lambda_i(\rho)\Bigr),\label{eq:renyi-inf}\end{equation}where $\|\cdot\|_{\mathrm{op}}$ denotes the spectral (operator) norm.

\subsection{Kernel-Induced Density Matrices and Kernel Covariance Operators}
\label{subsec:prelim:kernel}
Let $x_1,\ldots,x_n\in\mathcal{X}$ be data points and let
$k:\mathcal{X}\times\mathcal{X}\to\mathbb{R}$ be a positive semidefinite kernel
associated with a reproducing kernel Hilbert space (RKHS) $\mathcal{H}$ and
feature map $\phi:\mathcal{X}\to\mathcal{H}$. The associated kernel (Gram)
matrix is $K\in\mathbb{R}^{n\times n}$ with entries
\begin{equation}
K_{ij}\;:=\;k\bigl(x_i,x_j\bigr).
\label{eq:kernel-matrix}
\end{equation}
By positive semidefiniteness of $k$, we have $K\succeq 0$. We call a kernel function $k$ normalized if $k(x,x)=1$ for every $x\in\mathcal{X}$. Note that for every kernel function $k$, whether normalized or not, the normalized function $\widetilde{k}(x,y) = k(x,y)/\sqrt{k(x,x)k(y,y)}$ will be a valid normalized kernel function. Given the above definitions, we associate to $K$ the trace-normalized \emph{kernel-induced density matrix}
\begin{equation}
\rho_K \;:=\; \frac{K}{\mathrm{Tr}\bigl(K\bigr)},
\label{eq:kernel-density}
\end{equation}
where $\mathrm{Tr}\bigl(K\bigr)=n$ for the supposed normalized kernel function. Therefore, $\rho_K\succeq 0$ and
$\mathrm{Tr}\bigl(\rho_K\bigr)=1$ by default hold, and thus $\rho_K$ is a density matrix and the
matrix entropies defined in \eqref{eq:vne}--\eqref{eq:renyi-2} apply directly.
In particular, $S\bigl(\rho_K\bigr)$ provides a spectral notion of diversity or
effective dimensionality induced by the kernel similarity structure.

Equivalently, one may work at the operator level. The (uncentered) kernel
covariance operator $\mathcal{C}:\mathcal{H}\to\mathcal{H}$ associated with the
empirical distribution of $\{x_i\}_{i=1}^n$ is defined as
\begin{equation}
\mathcal{C}
\;:=\;
\frac{1}{n}\sum_{i=1}^n \phi\bigl(x_i\bigr)\otimes \phi\bigl(x_i\bigr),
\label{eq:kernel-cov-operator}
\end{equation}
where $\phi(x)\otimes\phi(x)$ denotes the rank-one operator
$h\mapsto \langle h,\phi(x)\rangle_{\mathcal{H}}\phi(x)$.
After trace normalization,
$\rho_{\mathcal{C}}:=\mathcal{C}/\mathrm{Tr}\bigl(\mathcal{C}\bigr)$
defines a density operator on $\mathcal{H}$.
The nonzero eigenvalues of $\rho_{\mathcal{C}}$ coincide with those of
$\rho_K$ (up to scaling), so von Neumann and R\'enyi entropies can be defined
consistently at either the matrix or operator level.

\section{Maximum von Neumann Entropy Principle}
\label{sec:maxvne}

\subsection{Game-Theoretic Setup and Generalized Entropy}
\label{subsec:setup}

We adopt a game-theoretic formulation following the maximum entropy principle  and its minimax interpretation developed for the classical entropy setting in \cite{GrunwaldDawid2004}. According to the problem formulation, 
Nature selects an unknown state $\rho$ from a feasible set $\Gamma$, while a
decision maker (DM) selects an action $q$ from an action set $\mathcal{Q}$ and
incurs a loss $L(\rho,q)$.

Throughout this section, we assume that $\Gamma$ is a nonempty, convex, and
compact subset of a finite-dimensional real \emph{Euclidean} space $(\mathsf{E},\langle\cdot,\cdot\rangle)$,
and that $\mathcal{Q}$ is a nonempty, convex, and compact subset of a
finite-dimensional Euclidean space $(\mathsf{F},\langle\cdot,\cdot\rangle)$.
We consider losses $L:\Gamma\times\mathcal{Q}\to\mathbb{R}$ that are continuous,
\emph{affine in $\rho$ on $\Gamma$} and \emph{convex in $q$ on $\mathcal{Q}$}.
In particular, the setting includes expected-loss models of the form
\begin{equation}
L(\rho,q)
=
\bigl\langle \rho,\, \ell(q)\bigr\rangle + c(q),
\label{eq:expected-loss-form}
\end{equation}
for continuous maps $\ell:\mathcal{Q}\to\mathsf{E}$ and $c:\mathcal{Q}\to\mathbb{R}$.

Given a state $\rho$, the minimal achievable loss defines the \emph{generalized
entropy} (Bayes risk)
\begin{equation}
H(\rho)
\;:=\;
\inf_{q \in \mathcal{Q}} L(\rho,q).
\label{eq:generalized-entropy}
\end{equation}
Note that the corresponding minimax value in this game theoretic setting is
\begin{equation}
V_\Gamma
\;:=\;
\inf_{q \in \mathcal{Q}} \ \sup_{\rho \in \Gamma} L(\rho,q).
\label{eq:minimax-value}
\end{equation}

%\subsection{Maximum Entropy Principle}\label{subsec:maxent}

\begin{definition}[Maximum Entropy Principle]
Given a feasible set $\Gamma$, the maximum (matrix-based) entropy principle selects
\begin{equation}
\rho^\star \in \arg\max_{\rho \in \Gamma} H(\rho),
\label{eq:maxent-def}
\end{equation}
where $H(\rho)$ is the generalized entropy in \eqref{eq:generalized-entropy}.
\end{definition}

\noindent
The definition is agnostic to the  entropy functional as different losses
induce different generalized entropies through \eqref{eq:generalized-entropy}.

\subsection{Minimax Interpretation of the Maximum Entropy Principle}
\label{subsec:minimax}

We now show that, under the assumptions discussed in the previous subsection, the maximum entropy principle
admits a precise minimax interpretation. In particular, the maximum generalized
entropy state corresponds to Nature's equilibrium strategy in a minimax game,
while the associated action is a robust Bayes decision rule.

\begin{theorem}[Linear-in-State Minimax Theorem]
\label{thm:general-minimax}
Let $\Gamma\subseteq\mathsf{E}$ and $\mathcal{Q}\subseteq\mathsf{F}$ be nonempty,
convex, and compact, and let $L:\Gamma\times\mathcal{Q}\to\mathbb{R}$ be
continuous, affine in $\rho$ on $\Gamma$, and convex in $q$ on $\mathcal{Q}$.
Define $H(\rho)$ and $V_\Gamma$ as in
\eqref{eq:generalized-entropy}--\eqref{eq:minimax-value}. Then:
\begin{enumerate}[leftmargin=*]
\item The minimax equality holds:
\begin{equation}
\inf_{q \in \mathcal{Q}} \ \sup_{\rho \in \Gamma} L(\rho,q)
\;=\;
\sup_{\rho \in \Gamma} \ \inf_{q \in \mathcal{Q}} L(\rho,q).
\label{eq:minimax-equality}
\end{equation}
\item There exists a saddle point $(\rho^\star,q^\star) \in \Gamma \times \mathcal{Q}$
such that
\begin{equation}
L(\rho,q^\star)
\;\le\;
L(\rho^\star,q^\star)
\;\le\;
L(\rho^\star,q),
\qquad
\forall\, \rho \in \Gamma,\ \forall\, q \in \mathcal{Q}.
\label{eq:saddle}
\end{equation}
\item The saddle-point state maximizes generalized entropy and the action is minimax-robust:
\begin{equation}
H(\rho^\star)
\;=\;
\sup_{\rho \in \Gamma} H(\rho)
\;=\;
V_\Gamma
\;=\;
L(\rho^\star,q^\star).
\label{eq:maxent-minimax}
\end{equation}
\end{enumerate}
\end{theorem}

\begin{proof}
The proof is deferred to the Appendix.
\end{proof}

\subsection{Specialization to von Neumann Entropy}
\label{subsec:vne}

We now specialize the framework to quantum states.
Let $\mathcal{D}_n:=\{\rho\succeq 0:\mathrm{Tr}(\rho)=1\}$ denote the set of
$n\times n$ density matrices. For technical convenience in the main text, we
work with the compact subset of strictly positive density matrices
\begin{equation}
\mathcal{D}_{n,\varepsilon}
\;:=\;
\Bigl\{\rho\in\mathcal{D}_n:\ \rho\succeq \varepsilon I\Bigr\},
\label{eq:D-eps}
\end{equation}
for a fixed $\varepsilon\in(0,1)$. We assume $\Gamma\subseteq\mathcal{D}_{n,\varepsilon}$
is convex and compact, and the DM action set is $\mathcal{Q}=\mathcal{D}_{n,\varepsilon}$.

Consider the \emph{quantum log loss}
\begin{equation}
L_{\log}(\rho,\sigma)
\;:=\;
-\mathrm{Tr}\bigl(\rho \log \sigma\bigr),
\qquad
\rho\in\Gamma,\ \sigma\in\mathcal{D}_{n,\varepsilon},
\label{eq:log-loss}
\end{equation}
which is well-defined and continuous on $\Gamma\times\mathcal{D}_{n,\varepsilon}$. In Appendix~B, we observe that the Bayes risk induced by
$L_{\log}$ coincides with the von Neumann entropy of every $\rho\in\Gamma$
\begin{equation}
\inf_{\sigma \in \mathcal{D}_{n,\varepsilon}} L_{\log}(\rho,\sigma)
\;=\;
S(\rho)
\;:=\;
-\mathrm{Tr}\bigl(\rho \log \rho\bigr).
\label{eq:vne-bayes}
\end{equation}

\begin{corollary}[Max-VNE as a Minimax Equilibrium]
\label{cor:maxvne}
Applying Theorem~\ref{thm:general-minimax} with the log loss $L_{\log}$, the
maximum von Neumann entropy principle
\begin{equation}
\rho^\star \in \arg\max_{\rho \in \Gamma} S(\rho)
\end{equation}
selects Nature's equilibrium state in the minimax game. Moreover, the
corresponding optimal action $\sigma^\star$ is a minimax-robust Bayes decision rule.
\end{corollary}

\begin{proof}
The proof is deferred to the Appendix.
\end{proof}

\subsection{Beyond von Neumann entropy: trace entropies and Bregman-type decompositions.}
The log loss in \eqref{eq:log-loss} admits the classical decomposition
\begin{equation}
L_{\log}(\rho,\sigma)=S(\rho)+D\bigl(\rho\Vert\sigma\bigr),
\label{eq:logloss-decomp-main}
\end{equation}
where $D(\rho\Vert\sigma)$ is the quantum relative entropy. This ``entropy plus divergence''
form is not unique to $S(\rho)$: a broad family of \emph{trace entropies}
can be obtained as Bayes risks of losses whose regret is a matrix Bregman divergence.

Concretely, let $f:(0,\infty)\to\mathbb{R}$ be convex and differentiable, and define the
trace entropy functional
\begin{equation}
H_f(\rho)
\;:=\;
-\mathrm{Tr}\bigl(f(\rho)\bigr),
\qquad \rho\in\mathcal{D}_{n,\varepsilon}.
\label{eq:Hf-main}
\end{equation}
For the associated ``$f$-loss'' $L_f(\rho,\sigma)$ defined in Appendix~C, one has the
exact regret decomposition
\begin{equation}
L_f(\rho,\sigma)
=
H_f(\rho) + D_f\bigl(\rho\Vert\sigma\bigr),
\label{eq:f-decomp-main}
\end{equation}
where $D_f(\rho\Vert\sigma)\ge 0$ is the matrix Bregman divergence generated by $f$.
In particular, $f(t)=t\log t$ recovers \eqref{eq:logloss-decomp-main},
while $f(t)=t^2$ gives $H_f(\rho)=-\mathrm{Tr}(\rho^2)$, which is equivalent to
R\'enyi-$2$ entropy through the monotone transform $S_2(\rho)=-\log \mathrm{Tr}(\rho^2)$. In the Appendix, we discuss the special cases.

\iffalse
More generally, for $\alpha>0$, R\'enyi-$\alpha$ entropy can be written as a monotone transform
of $\mathrm{Tr}\bigl(\rho^\alpha\bigr)$, and maximizing $S_\alpha(\rho)$ is equivalent to
optimizing $\mathrm{Tr}\bigl(\rho^\alpha\bigr)$ with the appropriate sign depending on $\alpha$.
Appendix~C formalizes the $f$-entropy/Bregman viewpoint and records sufficient conditions under
which the induced loss is convex in the action, allowing Theorem~\ref{thm:general-minimax} to
apply verbatim.
\fi

\section{Applications to Machine Learning}
\label{sec:applications}

We now illustrate how the Maximum von Neumann Entropy (Max-VNE) principle developed
in Section~\ref{sec:maxvne} applies to two representative machine learning problems
involving kernel-based representations under partial information.
Throughout this section, we assume a normalized positive semidefinite kernel
$k$ satisfying $k(x,x)=1$ for all input $x\in\mathcal{X}$.
Under this assumption, kernel covariance operators and kernel matrices admit a
natural interpretation as density matrices, without requiring additional trace
normalization.
%In high-dimensional learning problems, such kernel-induced operators provide a spectral representation whose eigen-directions capture latent structure in the data, making them a natural domain for entropy-based inference.

\subsection{Selecting a Kernel Mixture via Max-VNE}
\label{subsec:mixture}

In many learning settings, data can be represented through multiple normalized
embeddings, each inducing a kernel matrix.
Rather than fixing a single representation, we consider a constrained class of
convex kernel mixtures and apply the maximum von Neumann entropy (Max-VNE)
principle to \emph{select} a least-committal mixture consistent with prior
constraints.

To address this, consider \(M\) normalized embedding maps
\(
\phi_i:\mathcal{X}\to\mathbb{R}^{d_i}
\),
each inducing a kernel (i.e., inner product):
\begin{equation}
k_i(x,x')
\;:=\;
\bigl\langle \phi_i(x), \phi_i(x') \bigr\rangle,
\;\;
k_i(x,x)=1,\ \forall x\in\mathcal{X}.
\label{eq:base-kernels}
\end{equation}
Given samples \(x_1,\dots,x_n\), let \(K_i\in\mathbb{R}^{n\times n}\) denote the
corresponding kernel matrices, so that \(\mathrm{Tr}(K_i)=n\).

We form a combined representation via weighted concatenation.
For weights \(\alpha=[\alpha_1,\dots,\alpha_M]\in\Delta_M\), define
\begin{equation}
\phi_\alpha(x)
\;:=\;
\bigl[\sqrt{\alpha_1}\phi_1(x),\dots,\sqrt{\alpha_M}\phi_M(x)\bigr],
\label{eq:concat-embedding}
\end{equation}
which satisfies \(\|\phi_\alpha(x)\|_2=1\) for all \(x\).
The induced kernel is
\begin{equation}
k_\alpha(x,x')
\;=\;
\bigl\langle \phi_\alpha(x),\phi_\alpha(x') \bigr\rangle
\;=\;
\sum_{i=1}^M \alpha_i\, k_i(x,x').
\label{eq:kernel-mixture}
\end{equation}
Accordingly, the empirical kernel matrix is the convex combination with weights $[\alpha_1,\ldots ,\alpha_M]$
\begin{equation}
K(\alpha)
\;:=\;
\sum_{i=1}^M \alpha_i\, K_i,
\qquad
\alpha\in\mathcal{A}\subseteq\Delta_M,
\label{eq:gram-mixture}
\end{equation}
where \(\mathcal{A}\) is a convex set and a subset of the probability simplex $\Delta_M$.
Since \(\mathrm{Tr}(K(\alpha))=n\), the normalized matrix will be
\begin{equation}
\rho(\alpha)
\;:=\;
\frac{1}{n}\,K(\alpha)
\label{eq:rho-alpha}
\end{equation}
which is positive semidefinite with unit trace and can be viewed as a density matrix. The mixture class \eqref{eq:gram-mixture} induces the ambiguity set
\begin{equation}
\Gamma_{\mathrm{mix}}
:=
\bigl\{\rho(\alpha): \alpha\in\mathcal{A}\bigr\}.
\label{eq:Gamma-mix-new}
\end{equation}
Given only that the data representation lies in this class, we apply the Max-VNE
principle and select
\begin{equation}
\alpha^\star \in \arg\max_{\alpha\in\mathcal{A}} S\bigl(\rho(\alpha)\bigr),
\qquad
\rho^\star := \rho(\alpha^\star),
\label{eq:maxvne-kernel-mixture}
\end{equation}
where \(S(\cdot)\) denotes the von Neumann entropy. By the minimax interpretation in Section~\ref{sec:maxvne},
\(\rho^\star\) corresponds to Nature’s equilibrium density matrix within
\(\Gamma_{\mathrm{mix}}\).%, yielding a least-committal spectral representation consistent with the mixture constraints. The selected kernel can be used directly for downstream spectral analysis, including kernel PCA and similarity-based evaluation.

\subsection{Kernel Matrix and Covariance Completion via Max-VNE}
\label{subsec:completion}

We next consider the problem of completing a kernel matrix from partial observations.
Consider $x_1,\dots,x_n\in\mathcal{X}$ as data points to define the kernel matrix $K\in\mathbb{R}^{n\times n}$, and let
\begin{equation}
K_{ij} := k(x_i,x_j)
\end{equation}
denote the associated kernel matrix.
Under the normalization $k(x,x)=1$, we have $\mathrm{Tr}(K)=n$, and the scaled matrix
\begin{equation}
\rho := \frac{1}{n}K
\end{equation}
is a density matrix in $\mathcal{D}_n$.

Suppose that only a subset of kernel entries
\(
\{K_{ij}:(i,j)\in\Omega\}
\)
is observed, corresponding to linear constraints on $\rho$.
The set of density matrices consistent with the observations is
\begin{equation}
\Gamma_{\Omega}
:=
\Bigl\{
\rho\in\mathcal{D}_n :
\rho_{ij} = \tfrac{1}{n}K_{ij}
\ \text{for all } (i,j)\in\Omega
\Bigr\},
\label{eq:Gamma-omega}
\end{equation}
which is convex and compact. Applying the Max-VNE principle, we define the completed kernel matrix as
\begin{equation}
\rho^\star \in \arg\max_{\rho\in\Gamma_{\Omega}}\;  S(\rho).
\label{eq:maxvne-completion}
\end{equation}
The solution $\rho^\star$ is the maximum-entropy completion consistent with the
observed entries, assigning no additional spectral structure beyond what is
enforced by the constraints.
By Section~\ref{sec:maxvne}, $\rho^\star$ admits a minimax interpretation as the
robust equilibrium density matrix within $\Gamma_{\Omega}$.
The completed kernel matrix can then be used in downstream kernel methods such as
kernel PCA or spectral clustering.

\iffalse
\begin{remark}
In both applications, partial information about the data-generating process induces
an ambiguity set of kernel-induced density matrices.
The Max-VNE principle selects a canonical representative by maximizing entropy at
the level of kernel covariance operators or kernel matrices, whose eigenspaces
encode latent structure in high-dimensional data.
This matrix-based entropy perspective provides a principled alternative to applying
maximum entropy directly in the raw input space and leads to robust spectral
representations suitable for further analysis.
In the next section, we illustrate these ideas with numerical experiments on
synthetic mixtures and real datasets.
\end{remark}
\fi5-theory

\section{Numerical Results}

\subsection{Selecting Mixtures of Pretrained Embeddings via Max-VNE}

We evaluated the learned representation by the Maximum VNE principle on four standard image classification benchmarks: ImageNet~\cite{deng2009imagenet}, CIFAR-100~\cite{krizhevsky2009learning}, Describable Textures Dataset (DTD)~\cite{cimpoi2014describing}, and FGVC Aircraft~\cite{maji2013fine}.

\textbf{Models and Experimental Setup.}
Following the standard practice in the representation learning literature, we consider DINOv2~\cite{oquab2023dinov2}, SigLIP~\cite{zhai2023sigmoid}, UNICOM~\cite{an2023unicom}, and OpenCLIP~\cite{ilharco_gabriel_2021_5143773,radford2021learning} models.
For each pretrained embedding, we construct an RBF kernel and apply the Max-VNE principle to obtain the optimal mixture weights $\alpha^\star$ of the resulting kernel matrices. The learned weights are used to construct a mixed feature representation by scaling each embedding with its found weight and concatenation.

To evaluate whether the Max-VNE mixture yields a more informative representation for downstream tasks, we adopt a standard linear probing protocol. Specifically, it is performed using a single linear classifier trained on top of embeddings. Optimization is performed with Adam, a learning rate of $1\times10^{-3}$, batch size $256$, and $10$ training epochs. All experiments are conducted on two NVIDIA RTX~4090 GPUs.

\textbf{Results.}
Table~\ref{tab:lp_std} reports linear probing accuracy on all datasets. The Max-VNE mixed embedding outperforms every individual embedding over the four benchmark datasets.
\begin{figure}[t]
    \centering
    \includegraphics[width=\textwidth]{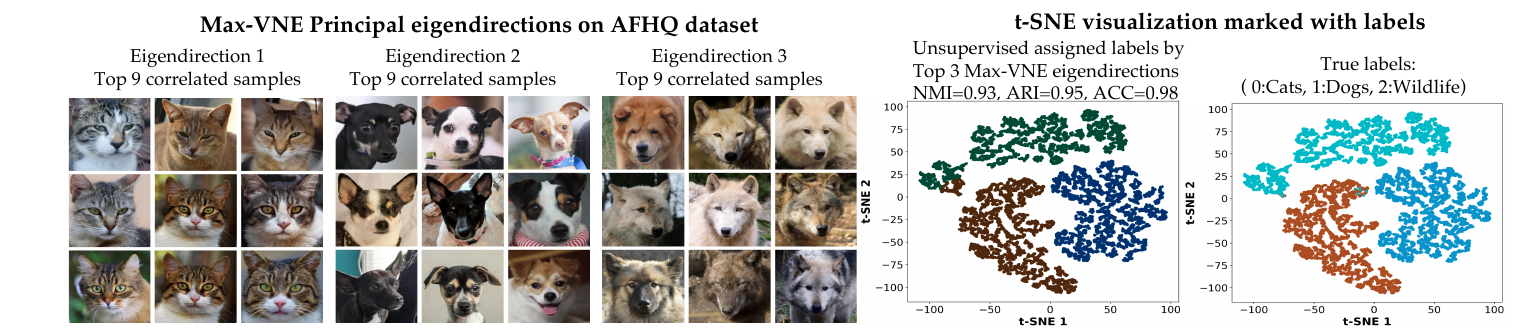}
    \caption{The experimental results of Max-VNE kernel completion from 10\% observed entries on the AFHQ dataset. The clustering performance evaluation of the recovered kernel matrix is performed using Normalized Mutual Information (NMI), Adjusted Rand Index (ARI), and Accuracy (ACC).}
    \label{fig:afhq_tsne}
\end{figure}

\begin{figure}
    \centering
    \includegraphics{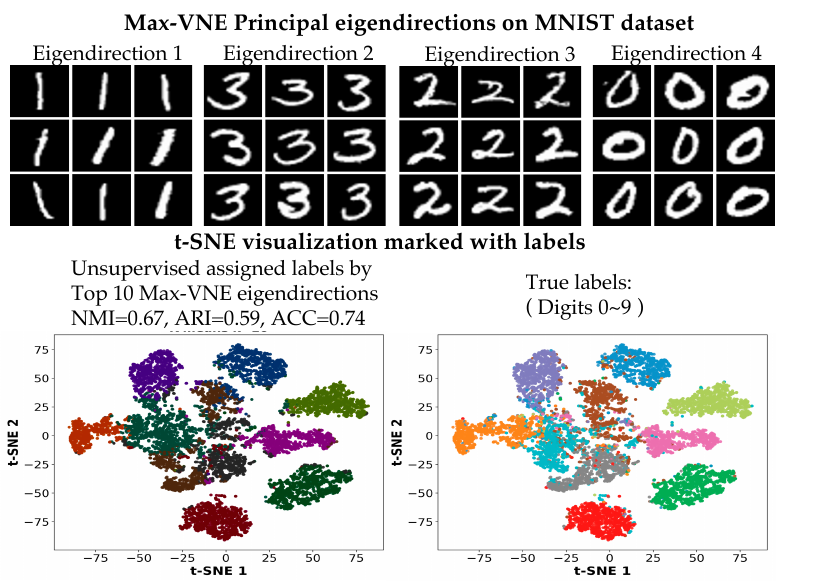}
    \caption{The experimental results of Fig.1. on the MNIST dataset.}
    \label{fig:mnist_tsne}
\end{figure}

\begin{table}[t]
\caption{linear probing accuracy of embeddings vs. Max-VNE-selected mixture of embeddings (\%, mean $\pm$ std over 5 runs).}
\label{tab:lp_std}
\centering
\begin{tabular}{lcccc}
\toprule

\textbf{Model} & \textbf{ImageNet} & \textbf{CIFAR100}& \textbf{DTD} & \textbf{Aircraft} \\
\midrule
OpenCLIP    &  $78.9 \pm 0.14$  & $86.3 \pm 0.02$ &  $77.7 \pm 0.23$ & $49.2 \pm 0.39$\\
SigLIP      &  $81.8 \pm 0.16$  & $83.3 \pm 0.04$ &  $80.3 \pm0.05$ &  $60.7 \pm 0.41$ \\
DINOv2      &  $80.0 \pm 0.19$  & $89.9 \pm 0.13$ &  $79.9 \pm0.22$ &  $58.9 \pm 0.63 $\\
UNICOM      &  $77.4 \pm 0.08$  & $85.0 \pm0.04 $ &  $71.6 \pm 0.53$ &  $55.4 \pm 0.85$\\
\midrule
Max-VNE       & $85.1 \pm 0.12$  & $91.5 \pm 0.06 $  &   $82.5 \pm 0.17 $  &  $65.1 \pm 0.48$  \\
\bottomrule
\end{tabular}
\vspace{-5mm}
\end{table}

\subsection{Kernel Matrix Completion via the Max-VNE Principle}
We conducted kernel completion experiments on three standard image datasets with different visual complexity: AFHQ~\cite{choi2020starganv2} (animal faces), MNIST~\cite{lecun1998mnist} (handwritten digits), and ImageNet-Dogs~\cite{Howard_Imagewoof_2019} (dog breeds from ImageNet).

\textbf{Experimental Setup.}
For each dataset, we construct a cosine similarity kernel matrix from the OpenCLIP~\cite{ilharco_gabriel_2021_5143773} image embeddings. To simulate partial observations, we randomly mask $90\%$ of the off-diagonal entries. Given the partially observed kernel, we apply the Max-VNE principle to recover a completed kernel matrix. Spectral clustering is then performed on the recovered kernel matrix.

\textbf{Results.}
Figures~\ref{fig:afhq_tsne},~\ref{fig:mnist_tsne} show the meaningful clustering results on the AFHQ and MNIST datasets. Quantitatively, 
we report NMI~\cite{mcdaid2013normalizedmutualinformationevaluate},
ARI~\cite{vinh2009information}, and clustering accuracy (ACC)~\cite{kuhn1955hungarian} in the figure. Overall, Max-VNE performs satisfactorily across datasets, indicating its effectiveness.  Results on the ImageNet-Dogs are included in the Appendix.

%\section{Conclusion}
%\input{sections/7-conclusion}

\newpage

\bibliographystyle{unsrt}
\bibliography{reference}

\clearpage
\newpage
\onecolumn
\begin{appendices}
\section{Proofs and Additional Theoretical Results}
%\appendices

\subsection{Proof of Theorem~\ref{thm:general-minimax}}
\label{app:general-minimax}

We prove Theorem~\ref{thm:general-minimax} under the stated assumptions. Throughout, $\Gamma$ and $\mathcal{Q}$ are nonempty, convex, and compact subsets
of finite-dimensional real vector spaces, and
$L:\Gamma\times\mathcal{Q}\to\mathbb{R}$ is continuous, affine in $\rho$, and
convex in $q$.

For each $q\in\mathcal{Q}$, define Nature’s best-response set
\[
\mathrm{BR}_{\Gamma}(q)
\;:=\;
\arg\max_{\rho\in\Gamma} L(\rho,q),
\]
and for each $\rho\in\Gamma$, define the decision maker’s best-response set
\[
\mathrm{BR}_{\mathcal{Q}}(\rho)
\;:=\;
\arg\min_{q\in\mathcal{Q}} L(\rho,q).
\]
By continuity of $L$ and compactness of $\Gamma$ and $\mathcal{Q}$, both sets are
nonempty and compact. Affinity of $L(\cdot,q)$ implies that
$\mathrm{BR}_{\Gamma}(q)$ is convex, and convexity of $L(\rho,\cdot)$ implies that
$\mathrm{BR}_{\mathcal{Q}}(\rho)$ is convex.

Consider the product set $K:=\Gamma\times\mathcal{Q}$ and the correspondence
$T:K\rightrightarrows K$ defined by
\[
T(\rho,q)
\;:=\;
\mathrm{BR}_{\Gamma}(q)\times \mathrm{BR}_{\mathcal{Q}}(\rho).
\]
The correspondence $T$ has nonempty, convex, and compact values.
Moreover, since $L$ is continuous, standard arguments show that the graph of $T$
is closed: if $(\rho_k,q_k)\to(\rho,q)$ and
$(\tilde\rho_k,\tilde q_k)\in T(\rho_k,q_k)$ with
$(\tilde\rho_k,\tilde q_k)\to(\tilde\rho,\tilde q)$, then continuity of $L$
implies $(\tilde\rho,\tilde q)\in T(\rho,q)$.
In finite-dimensional spaces, closed graph and compact values imply upper
hemicontinuity.

\begin{lemma}[Existence of a saddle point]
There exists $(\rho^\star,q^\star)\in\Gamma\times\mathcal{Q}$ such that
\[
\rho^\star\in \mathrm{BR}_{\Gamma}(q^\star)
\qquad\text{and}\qquad
q^\star\in \mathrm{BR}_{\mathcal{Q}}(\rho^\star).
\]
\end{lemma}

\begin{proof}
Since $K$ is nonempty, compact, and convex, and $T$ is upper hemicontinuous with
nonempty, compact, and convex values, Kakutani’s fixed point theorem applies and
yields a fixed point $(\rho^\star,q^\star)\in K$ satisfying
$(\rho^\star,q^\star)\in T(\rho^\star,q^\star)$.
\end{proof}

We now complete the proof of Theorem~\ref{thm:general-minimax}.
From the fixed-point property,
$\rho^\star$ maximizes $L(\rho,q^\star)$ over $\Gamma$ and
$q^\star$ minimizes $L(\rho^\star,q)$ over $\mathcal{Q}$.
Equivalently,
\[
L(\rho,q^\star)
\;\le\;
L(\rho^\star,q^\star)
\;\le\;
L(\rho^\star,q)
\qquad
\forall\,\rho\in\Gamma,\ \forall\,q\in\mathcal{Q},
\]
which establishes the saddle-point property and proves item~2 of the theorem.

The minimax equality follows immediately.
Weak duality always gives
\[
\sup_{\rho\in\Gamma}\inf_{q\in\mathcal{Q}} L(\rho,q)
\;\le\;
\inf_{q\in\mathcal{Q}}\sup_{\rho\in\Gamma} L(\rho,q).
\]
Using the saddle inequalities,
\[
\sup_{\rho\in\Gamma} L(\rho,q^\star)
=
L(\rho^\star,q^\star)
=
\inf_{q\in\mathcal{Q}} L(\rho^\star,q),
\]
which implies equality and proves item~1.

Finally, define the generalized entropy
$H(\rho):=\inf_{q\in\mathcal{Q}} L(\rho,q)$.
Since $q^\star$ minimizes $L(\rho^\star,\cdot)$,
\[
H(\rho^\star)=L(\rho^\star,q^\star).
\]
For any $\rho\in\Gamma$,
\[
H(\rho)
=
\inf_{q}L(\rho,q)
\le
L(\rho,q^\star)
\le
L(\rho^\star,q^\star)
=
H(\rho^\star),
\]
so $\rho^\star$ maximizes $H(\rho)$ over $\Gamma$.
Moreover,
\[
V_\Gamma
=
\inf_{q}\sup_{\rho}L(\rho,q)
=
\sup_{\rho}L(\rho,q^\star)
=
L(\rho^\star,q^\star)
=
H(\rho^\star),
\]
which proves item~3 and completes the proof.

\subsection{Details for the von Neumann Entropy Specialization}
\label{app:vne}

This appendix provides details for Corollary~\ref{cor:maxvne}, verifying that the
quantum log loss satisfies the assumptions of Theorem~\ref{thm:general-minimax}
and that its induced generalized entropy coincides with the von Neumann entropy.

Let $\mathcal{D}_n:=\{\rho\succeq 0:\mathrm{Tr}(\rho)=1\}$ denote the set of
$n\times n$ density matrices, and let $\Gamma\subseteq\mathcal{D}_n$ be a convex
and compact feasible set. We consider the quantum log loss
\[
L_{\log}(\rho,\sigma)
\;:=\;
-\mathrm{Tr}\bigl(\rho\log\sigma\bigr),
\]
defined for $\sigma\in\mathcal{D}_n$ with $\sigma\succ 0$. For technical convenience,
we restrict the action set to
$\Sigma_\varepsilon:=\{\sigma\in\mathcal{D}_n:\sigma\succeq \varepsilon I\}$ for some
$\varepsilon>0$, which ensures that $\log\sigma$ is finite and continuous.

For each fixed $\sigma$, the map $\rho\mapsto L_{\log}(\rho,\sigma)$ is affine, since
the trace is linear. For each fixed $\rho$, the map
$\sigma\mapsto L_{\log}(\rho,\sigma)$ is convex on $\Sigma_\varepsilon$, because the
scalar function $x\mapsto -\log x$ is operator convex on $(0,\infty)$ and trace
preserves the L\"owner order when left-multiplied by $\rho\succeq 0$. Continuity of
$L_{\log}$ on $\Gamma\times\Sigma_\varepsilon$ follows from continuity of the matrix
logarithm on $\Sigma_\varepsilon$. Thus all assumptions of
Theorem~\ref{thm:general-minimax} are satisfied.

We now identify the generalized entropy induced by the log loss.
Define the von Neumann entropy
$S(\rho):=-\mathrm{Tr}\bigl(\rho\log\rho\bigr)$ and the quantum relative entropy
\[
D(\rho\Vert\sigma)
\;:=\;
\mathrm{Tr}\bigl(\rho(\log\rho-\log\sigma)\bigr),
\]
with the convention $D(\rho\Vert\sigma)=+\infty$ if
$\mathrm{supp}(\rho)\nsubseteq\mathrm{supp}(\sigma)$.
A direct calculation yields the decomposition
\[
L_{\log}(\rho,\sigma)
=
S(\rho)+D(\rho\Vert\sigma),
\]
whenever $D(\rho\Vert\sigma)$ is finite.

The quantum relative entropy is nonnegative by Klein’s inequality, with equality
if and only if $\rho=\sigma$. Consequently, for every $\rho\in\mathcal{D}_n$,
\[
\inf_{\sigma\in\mathcal{D}_n} L_{\log}(\rho,\sigma)
=
S(\rho),
\]
and the infimum is attained at $\sigma=\rho$.

\begin{proof}[Proof of Corollary~\ref{cor:maxvne}]
Applying Theorem~\ref{thm:general-minimax} to the game in which Nature selects
$\rho\in\Gamma$ and the decision maker selects $\sigma\in\Sigma_\varepsilon$ with
loss $L_{\log}(\rho,\sigma)$, we obtain a saddle point
$(\rho^\star,\sigma^\star)$ and minimax equality. Since the generalized entropy
induced by $L_{\log}$ coincides with the von Neumann entropy, the maximum generalized
entropy principle reduces to the maximum von Neumann entropy principle. The
corresponding action $\sigma^\star$ is minimax-robust.
\end{proof}

\subsection{Trace Entropies and an Entropy--Divergence Decomposition}
\label{app:trace-bregman}

This appendix records a general entropy plus divergence identity for a broad
family of trace entropies, and briefly relates it to R\'enyi-$\alpha$ objectives.

In the discussion, we fix $\varepsilon\in(0,1/n)$ and recall
$\mathcal{D}_{n,\varepsilon}:=\{\rho\in\mathcal{D}_n:\rho\succeq \varepsilon I\}$.
Let $f:(0,\infty)\to\mathbb{R}$ be convex and continuously differentiable.
For $\rho\in\mathcal{D}_{n,\varepsilon}$, define the trace-entropy functional
\begin{equation}
H_f(\rho)
\;:=\;
-\mathrm{Tr}\bigl(f(\rho)\bigr).
\label{eq:Hf-app}
\end{equation}
For $\rho,\sigma\in\mathcal{D}_{n,\varepsilon}$, define the associated matrix
Bregman divergence
\begin{equation}
D_f\bigl(\rho\Vert\sigma\bigr)
\;:=\;
\mathrm{Tr}\Bigl(f(\rho)-f(\sigma)-f'(\sigma)\bigl(\rho-\sigma\bigr)\Bigr).
\label{eq:bf-def}
\end{equation}
Nonnegativity follows from convexity of the spectral function
$F(\rho):=\mathrm{Tr}(f(\rho))$ on Hermitian matrices: for convex differentiable
$F$ one has
\[
F(\rho)\ge F(\sigma)+\mathrm{Tr}\bigl(f'(\sigma)(\rho-\sigma)\bigr),
\]
which is equivalent to $D_f(\rho\Vert\sigma)\ge 0$.

Next, define the corresponding loss
\begin{equation}
L_f(\rho,\sigma)
\;:=\;
-\mathrm{Tr}\Bigl(f(\sigma)+f'(\sigma)\bigl(\rho-\sigma\bigr)\Bigr),
\qquad
\rho,\sigma\in\mathcal{D}_{n,\varepsilon}.
\label{eq:Lf-def}
\end{equation}
For each fixed $\sigma$, the map $\rho\mapsto L_f(\rho,\sigma)$ is affine.

\begin{proposition}
\label{prop:entropy-bregman}
For all $\rho,\sigma\in\mathcal{D}_{n,\varepsilon}$, the following holds
\begin{equation}
L_f(\rho,\sigma)
=
H_f(\rho) + D_f\bigl(\rho\Vert\sigma\bigr).
\label{eq:Lf-decomp}
\end{equation}
Consequently, the following is true where the infimum is attained at $\sigma=\rho$:
\begin{equation}
\inf_{\sigma\in\mathcal{D}_{n,\varepsilon}} L_f(\rho,\sigma)
=
H_f(\rho)
\label{eq:Lf-bayes}
\end{equation}
\end{proposition}

\begin{proof}
By \eqref{eq:Hf-app}, \eqref{eq:bf-def}, and \eqref{eq:Lf-def},
\[
H_f(\rho)+D_f(\rho\Vert\sigma)
=
-\mathrm{Tr}\bigl(f(\rho)\bigr)
+
\mathrm{Tr}\Bigl(f(\rho)-f(\sigma)-f'(\sigma)\bigl(\rho-\sigma\bigr)\Bigr)
=
-\mathrm{Tr}\Bigl(f(\sigma)+f'(\sigma)\bigl(\rho-\sigma\bigr)\Bigr)
=
L_f(\rho,\sigma).
\]
The Bayes-risk identity \eqref{eq:Lf-bayes} follows from $D_f(\rho\Vert\sigma)\ge 0$,
with equality at $\sigma=\rho$.
\end{proof}

\paragraph{Two special cases.}
For $f(t)=t\log t$, we have
$H_f(\rho)=-\mathrm{Tr}\bigl(\rho\log\rho\bigr)=S(\rho)$ and
$D_f(\rho\Vert\sigma)=D(\rho\Vert\sigma)$ (quantum relative entropy), and
\eqref{eq:Lf-decomp} recovers the standard identity
$-\mathrm{Tr}(\rho\log\sigma)=S(\rho)+D(\rho\Vert\sigma)$.
For $f(t)=t^2$, we obtain $H_f(\rho)=-\mathrm{Tr}(\rho^2)$ and
$D_f(\rho\Vert\sigma)=\mathrm{Tr}\bigl((\rho-\sigma)^2\bigr)$, yielding the quadratic
(order-$2$) surrogate used in the main text.

\paragraph{Relation to R\'enyi-$\alpha$.}
For $\alpha>0$, $\alpha\neq 1$, the (Petz) R\'enyi-$\alpha$ entropy is
\begin{equation}
S_\alpha(\rho)
\;:=\;
\frac{1}{1-\alpha}\log \mathrm{Tr}\bigl(\rho^\alpha\bigr).
\label{eq:renyi}
\end{equation}
Since $\log(\cdot)$ is increasing and $\frac{1}{1-\alpha}$ changes sign at $\alpha=1$,
maximizing $S_\alpha(\rho)$ over a feasible set $\Gamma$ is equivalent to maximizing
$\mathrm{Tr}(\rho^\alpha)$ for $0<\alpha<1$ and to minimizing $\mathrm{Tr}(\rho^\alpha)$
for $\alpha>1$. Accordingly, one may view maximum R\'enyi-$\alpha$ selection through
a trace-power surrogate $H_f(\rho)=-\mathrm{Tr}(f(\rho))$ with $f(t)=t^\alpha$ (up to
a sign choice depending on $\alpha$), to which Proposition~\ref{prop:entropy-bregman}
applies.

\subsection{Linear Constraints and Equalizer Solutions}
\label{app:linear-constraints}

This appendix gives a simple sufficient condition under which the maximum
von Neumann entropy (Max-VNE) solution satisfies an \emph{equalizer} property,
i.e., the log loss against the selected action is constant over the feasible set.

Let $A_1,\dots,A_m\in\mathbb{S}^n$ and $\tau\in\mathbb{R}^m$, and define the affine
constraint set
\[
\Gamma_{\tau}
:=
\Bigl\{\rho\in\mathcal{D}_n:\ \mathrm{Tr}\bigl(\rho A_j\bigr)=\tau_j,\ j=1,\dots,m\Bigr\}.
\]

\begin{theorem}
\label{thm:gibbs-equalizer}
Assume $\Gamma_{\tau}\neq\emptyset$ and there exists a full-rank
$\rho_{\tau}\in\Gamma_{\tau}$ and scalars $c,\beta_1,\dots,\beta_m$ such that
\begin{equation}
\log \rho_{\tau}
=
cI+\sum_{j=1}^m \beta_j A_j.
\label{eq:gibbs-form}
\end{equation}
Then for the quantum log loss $L_{\log}(\rho,\sigma)=-\mathrm{Tr}\bigl(\rho\log\sigma\bigr)$:
\begin{enumerate}
\item  The value $L_{\log}(\rho,\rho_{\tau})$ is constant over all
$\rho\in\Gamma_{\tau}$.
\item $\rho_{\tau}\in\arg\max_{\rho\in\Gamma_{\tau}} S(\rho)$.
\item  $(\rho_{\tau},\rho_{\tau})$ is a saddle point of the minimax game
$\inf_{\sigma\in\mathcal{D}_n}\sup_{\rho\in\Gamma_{\tau}} L_{\log}(\rho,\sigma)$, and hence
$\rho_{\tau}$ is minimax-optimal.
\end{enumerate}
\end{theorem}

\begin{proof}
For any $\rho\in\Gamma_{\tau}$, using \eqref{eq:gibbs-form} and the constraints
$\mathrm{Tr}(\rho A_j)=\tau_j$,
\[
L_{\log}(\rho,\rho_{\tau})
=
-\mathrm{Tr}\bigl(\rho\log\rho_{\tau}\bigr)
=
-\mathrm{Tr}\Bigl(\rho\Bigl(cI+\sum_{j=1}^m\beta_j A_j\Bigr)\Bigr)
=
-\Bigl(c+\sum_{j=1}^m\beta_j\tau_j\Bigr),
\]
which is independent of $\rho$. This proves item~1.

Next, we recall the standard decomposition for $\rho\in\mathcal{D}_n$ and full-rank
$\sigma\in\mathcal{D}_n$:
\[
L_{\log}(\rho,\sigma)
=
S(\rho)+D(\rho\Vert\sigma),
\]
where $S(\rho)=-\mathrm{Tr}(\rho\log\rho)$ is the von Neumann entropy and
$D(\rho\Vert\sigma)=\mathrm{Tr}\bigl(\rho(\log\rho-\log\sigma)\bigr)\ge 0$ is the
quantum relative entropy. Applying this with $\sigma=\rho_{\tau}$ gives, for all
$\rho\in\Gamma_{\tau}$,
\[
L_{\log}(\rho,\rho_{\tau})
=
S(\rho)+D(\rho\Vert\rho_{\tau})
\ge
S(\rho).
\]
By item~1, $L_{\log}(\rho,\rho_{\tau})=L_{\log}(\rho_{\tau},\rho_{\tau})=S(\rho_{\tau})$,
hence $S(\rho)\le S(\rho_{\tau})$ for all $\rho\in\Gamma_{\tau}$, proving item~2.

Finally, item~1 yields
\[
\sup_{\rho\in\Gamma_{\tau}} L_{\log}(\rho,\rho_{\tau})
=
L_{\log}(\rho_{\tau},\rho_{\tau})
=
S(\rho_{\tau}).
\]
On the other hand, for any $\sigma\in\mathcal{D}_n$,
\[
\sup_{\rho\in\Gamma_{\tau}} L_{\log}(\rho,\sigma)
\ge
L_{\log}(\rho_{\tau},\sigma)
=
S(\rho_{\tau})+D(\rho_{\tau}\Vert\sigma)
\ge
S(\rho_{\tau}),
\]
so $\rho_{\tau}$ attains the minimax value and $(\rho_{\tau},\rho_{\tau})$ satisfies
the saddle inequalities. This proves item~3.
\end{proof}

\subsection{Maximum Conditional von Neumann Entropy and a Minimax Interpretation}
\label{app:conditional-maxvne}

This appendix develops a conditional analogue of the maximum entropy principle in the
cq setting, where Nature selects matrix-valued conditional states.
The presentation parallels the minimax viewpoint of \cite{FarniaTse2016Minimax},
with the key difference that conditional distributions are replaced by conditional
density matrices.

Let $\mathcal{X}$ be a finite alphabet and let $p\in\Delta(\mathcal{X})$ be a fixed
distribution known to the decision maker (DM). Fix $\varepsilon\in(0,1/n)$ and define
$\mathcal{D}_{n,\varepsilon}:=\{\rho\succeq 0:\mathrm{Tr}(\rho)=1,\ \rho\succeq\varepsilon I\}$.
A cq state is specified by $\rho=\{\rho_x\}_{x\in\mathcal{X}}$ with
$\rho_x\in\mathcal{D}_{n,\varepsilon}$ and corresponds to the block-diagonal matrix
\[
\rho_{XQ}
=
\sum_{x\in\mathcal{X}} p(x)\,|x\rangle\langle x|\otimes\rho_x .
\]
Let $\Gamma$ be a nonempty, convex, and compact set of such cq states.
The DM selects an action $\sigma=\{\sigma_x\}_{x\in\mathcal{X}}$ with
$\sigma_x\in\mathcal{D}_{n,\varepsilon}$, i.e. the action set is
$\mathcal{Q}=(\mathcal{D}_{n,\varepsilon})^{|\mathcal{X}|}$. We consider the conditional log loss defined as:
\begin{equation}
L_{\log}(\rho,\sigma)
\;:=\;
\sum_{x\in\mathcal{X}} p(x)\,
\bigl[-\mathrm{Tr}\bigl(\rho_x\log\sigma_x\bigr)\bigr],
\label{eq:cond-log-loss}
\end{equation}
which is finite and continuous on $\Gamma\times\mathcal{Q}$ by the restriction
$\sigma_x\succeq\varepsilon I$.

\paragraph{Loss-induced conditional entropy.}
For a general loss $L(\rho,\sigma)$ on $\Gamma\times\mathcal{Q}$, define the generalized
(conditional) entropy and minimax value as
\[
H(\rho):=\inf_{\sigma\in\mathcal{Q}} L(\rho,\sigma),
\qquad
V_\Gamma:=\inf_{\sigma\in\mathcal{Q}}\sup_{\rho\in\Gamma} L(\rho,\sigma).
\]
For the conditional log loss, the minimization decomposes over $x$:
\[
\inf_{\sigma\in\mathcal{Q}} L_{\log}(\rho,\sigma)
=
\sum_{x\in\mathcal{X}} p(x)\,
\inf_{\sigma_x\in\mathcal{D}_{n,\varepsilon}}
\bigl[-\mathrm{Tr}(\rho_x\log\sigma_x)\bigr].
\]
Using the standard identity
\[
-\mathrm{Tr}(\rho_x\log\sigma_x)=S(\rho_x)+D(\rho_x\Vert\sigma_x),
\]
where $S(\rho_x):=-\mathrm{Tr}(\rho_x\log\rho_x)$ and $D(\cdot\Vert\cdot)\ge 0$ is the
quantum relative entropy, we obtain
\begin{equation}
H(\rho)
=
\sum_{x\in\mathcal{X}} p(x)\,S(\rho_x)
=: S(Q|X)_\rho,
\label{eq:cond-vne}
\end{equation}
and the infimum is attained at $\sigma_x=\rho_x$ for all $x$.
Thus, the generalized entropy induced by the conditional log loss coincides with the
conditional von Neumann entropy for cq states.

\paragraph{Maximum conditional entropy and minimax optimality.}
Consider the minimax game $\inf_{\sigma\in\mathcal{Q}}\sup_{\rho\in\Gamma}L_{\log}(\rho,\sigma)$.
We aim to identify $\rho=\{\rho_x\}$ with an element of the finite-dimensional product space
$\bigl(\mathbb{S}^n\bigr)^{|\mathcal{X}|}$, and equip it with the weighted pairing
\[
\bigl\langle \rho,\rho'\bigr\rangle
:=
\sum_{x\in\mathcal{X}} p(x)\,\mathrm{Tr}\bigl(\rho_x\rho'_x\bigr),
\]
and similarly for $\sigma$. Under this identification, $\Gamma$ and $\mathcal{Q}$ remain
convex and compact, and $L_{\log}$ is affine in $\rho$ and convex in $\sigma$.
Applying Theorem~\ref{thm:general-minimax} to this product-space game yields minimax
equality, existence of a saddle point $(\rho^\star,\sigma^\star)$, and the entropy
characterization
\[
\rho^\star \in \arg\max_{\rho\in\Gamma} H(\rho)
=
\arg\max_{\rho\in\Gamma} S(Q|X)_\rho.
\]
In particular, \emph{maximizing conditional von Neumann entropy} selects Nature's
equilibrium state in the minimax game, and the corresponding action $\sigma^\star$ is
minimax-robust.

\begin{remark}
For the quadric (order-2) special case, we define the quadratic conditional loss
\[
L_{2}(\rho,\sigma)
:=
\sum_{x\in\mathcal{X}} p(x)\,
\bigl[\mathrm{Tr}(\sigma_x^2)-2\,\mathrm{Tr}(\rho_x\sigma_x)\bigr].
\]
Completing the square gives, for each $x$,
\[
\mathrm{Tr}(\sigma_x^2)-2\,\mathrm{Tr}(\rho_x\sigma_x)
=
\mathrm{Tr}\bigl((\sigma_x-\rho_x)^2\bigr)-\mathrm{Tr}(\rho_x^2).
\]
Therefore, the Bayes risk is
\[
\inf_{\sigma\in\mathcal{Q}} L_2(\rho,\sigma)
=
-\sum_{x\in\mathcal{X}} p(x)\,\mathrm{Tr}(\rho_x^2),
\]
where the optimal value is attained at $\sigma_x=\rho_x$.
Since $S_2(\rho_x)=-\log\mathrm{Tr}(\rho_x^2)$, maximizing
$\sum_x p(x)S_2(\rho_x)$ is equivalent to maximizing the generalized entropy above,
and the same minimax interpretation applies.
\end{remark}
\section{Additional Related Work}

Beyond the works discussed in the main text, a broader literature studies entropy- and information-theoretic principles for learning, inference, and evaluation in modern machine learning. In particular, entropy-based objectives have been used to evaluate and compare generative models by quantifying coverage and diversity of learned distributions. Information-theoretic evaluations of multi-modal and generative models using entropy and mutual-information-inspired criteria are developed in \cite{jalali2023information}, while interpretable analyses of entropy-based novelty and diversity measures are provided in \cite{zhang2024interpretable}. Related reference-free and scalable evaluation frameworks for generative models further emphasize spectral and kernel-based information measures \cite{ospanov2024towards}.

A growing body of work extends von Neumann entropy and kernel-based information measures to prompt-aware and conditional generative modeling. Prompt-aware diversity and novelty guidance using kernel entropy and R\'enyi-type objectives is studied in \cite{jalalisparke}, while prompt-aware diversity evaluation via Schur-complement decompositions of embedding covariance structures is developed in \cite{ospanov2025scendi}. An information-theoretic treatment of diversity evaluation in prompt-based generative models is further explored in \cite{jalali2025information}. These works highlight the importance of conditional and prompt-dependent entropy notions, complementing the conditional and minimax perspectives discussed in the main text.

Spectral entropy measures also arise in graph and network analysis. Von Neumann graph entropy defines an entropy functional on normalized graph Laplacians and has been used to quantify graph complexity, structural change, and dynamics, together with scalable approximations for large graphs \cite{Chen2018VNGE}. This line of work reinforces the idea of treating normalized positive semidefinite matrices as density-matrix-like objects and studying entropy directly at the spectral level.

Finally, entropy and diversity principles naturally connect to mixture modeling and aggregation of multiple models or data sources. Diversity-aware mixture selection and aggregation methods appear in learning and decision-making contexts, including bandit-based approaches that aim to construct mixtures that are more diverse than any single component \cite{rezaei2024more}. These approaches are complementary to the maximum von Neumann entropy principle studied in this paper, which provides a static, information-theoretic justification for selecting mixtures and kernel representations under partial information.

\section{Additional Numerical Results on MAX-VNE-Selected Mixture of Embeddings}

\subsection{Implementation Details}
To select the mixture weights, we construct four RBF kernel matrices on a fixed subset of the ImageNet-1K training set, corresponding to OpenCLIP, DINOv2, SigLIP, and UNICOM embeddings.
Specifically, we randomly sample $200$ classes and select $50$ images per class, resulting in $10{,}000$ images in total.
The bandwidth parameter of the RBF kernel is selected individually for each model.
In particular, we choose the bandwidth such that the resulting kernel matrices have comparable von Neumann entropy values when considered independently.
This design choice avoids biasing the mixture optimization toward a specific model due to kernel sharpness or scale effects.
The mixture weights $\alpha$ are obtained by maximizing the von Neumann entropy using projected gradient ascent (PGA).
At each iteration, gradients with respect to $\alpha$ are computed analytically, followed by an Euclidean projection onto the simplex.
The optimization is initialized with uniform weights and optimized using a fixed learning rate of $0.1$ for at most $100$ iterations.

\subsubsection{Additional Linear Probe Results on Mixture of Linear Kernel via Max-VNE}
In addition to the RBF-kernel-based mixture results reported in the main text, we further evaluate Max-VNE-selected mixtures obtained using linear kernels.
Specifically, for each pretrained embedding, we construct a linear kernel and apply the same Max-VNE optimization procedure to select the mixture weights $\alpha^\star$.
The resulting mixed representation is formed by scaling each embedding with its corresponding weight and concatenating them.
Linear probing is then performed using exactly the same evaluation protocol as in the main text, including the same optimizer, learning rate, batch size, and number of training epochs.
The additional results are reported in Table~\ref{tab:lp_std_linear_kernel}, showing that the Max-VNE-selected mixture based on linear kernels achieves competitive performance, indicating that the proposed mixture selection strategy is not sensitive to the specific kernel choice.

\begin{table}[h]
\caption{linear probing accuracy of embeddings vs. Max-VNE-selected mixture of embeddings (\%, mean $\pm$ std over 5 runs).}
\label{tab:lp_std_linear_kernel}
\centering
\begin{tabular}{lcccc}
\toprule
\textbf{Model} & \textbf{ImageNet} & \textbf{CIFAR100}& \textbf{DTD} & \textbf{Aircraft} \\
\midrule
OpenCLIP    &  $78.9 \pm 0.14$  & $86.3 \pm 0.02$ &  $77.7 \pm 0.23$ & $49.2 \pm 0.39$\\
SigLIP      &  $81.8 \pm 0.16$  & $83.3 \pm 0.04$ &  $80.3 \pm0.05$ &  $60.7 \pm 0.41$ \\
DINOv2      &  $80.0 \pm 0.19$  & $89.9 \pm 0.13$ &  $79.9 \pm0.22$ &  $58.9 \pm 0.63 $\\
UNICOM      &  $77.4 \pm 0.08$  & $85.0 \pm0.04 $ &  $71.6 \pm 0.53$ &  $55.4 \pm 0.85$\\
\midrule
Max-VNE     &  $83.1 \pm 0.09$  & $91.2 \pm 0.09$ &  $82.3 \pm 0.13$  & $63.2 \pm 0.37$ \\
\bottomrule
\end{tabular}
\end{table}

\subsection{Additional Numerical Results for MAX-VNE  Kernel Completion}

\subsubsection{Implementation Details}
In the main experiments, we implement the Max-VNE principle using an order-2 R\'enyi entropy objective, which corresponds to minimizing the purity of the recovered density matrix.
To ensure scalability to large datasets, we parameterize the completed kernel using a low-rank factorization with a learnable matrix $U \in \mathbb{R}^{n \times r}$, which guarantees positive semidefiniteness by construction.
This factorization avoids explicit optimization over full $n \times n$ matrices and enables efficient computation.
The optimization problem is solved using gradient-based methods.
Specifically, we optimize the low-rank factors with rank $r=50$ using Adam with a fixed learning rate of $1\times10^{-3}$ until convergence.
In the ablation study, we further report results obtained by maximizing the von Neumann entropy, as well as by varying the factorization rank and the observation ratio.

\begin{figure}
    \centering
    \includegraphics[width=\textwidth]{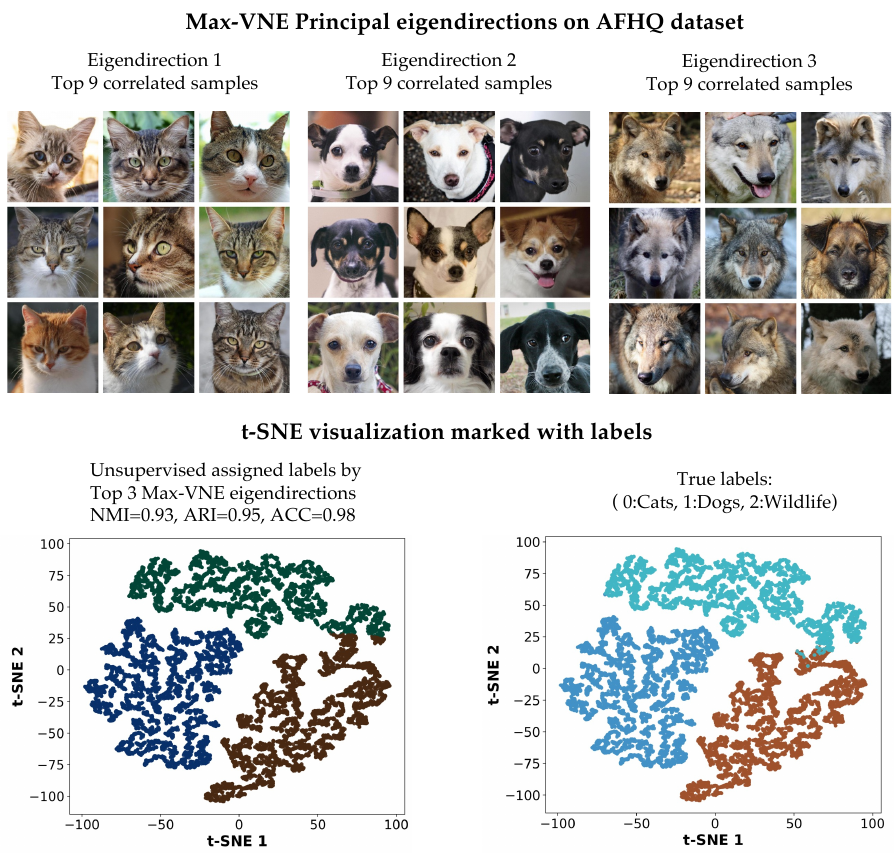}
    \caption{The experimental results of Max-VNE kernel completion from 10\% observed entries on the AFHQ dataset. The clustering performance evaluation of the recovered kernel matrix is performed using Normalized Mutual Information (NMI), Adjusted Rand Index (ARI), and Accuracy (ACC).}\vspace{-6mm}

    \label{fig:VNE_afhq_tsne}
\end{figure}

\subsubsection{Additional Results of Maximizing \emph{Von Neumann Entropy}}
In addition to the order-2 Rényi entropy objective used in the main text, we report supplementary results obtained by directly maximizing the von Neumann entropy for kernel matrix completion.
Following the same experimental setting, the kernel matrix is parameterized in a low-rank form as $\mathbf{K}=\mathbf{U}\mathbf{U}^\top$ and normalized to have unit trace.
The von Neumann entropy is optimized using a log-determinant surrogate defined on the resulting low-dimensional covariance, which avoids explicit eigendecomposition of the full kernel.
As shown in Figure~\ref{fig:VNE_afhq_tsne}, the recovered kernels yield clustering results on AFHQ that are qualitatively and quantitatively comparable to those obtained with the order-2 objective, indicating that the observed behavior is robust to the choice of entropy formulation.

\subsubsection{Ablation Study}

\paragraph{Additional Results on the More Challenge Dataset}

We further evaluate the proposed kernel matrix completion approach on the ImageNet-Dogs dataset~\cite{Howard_Imagewoof_2019}, which is known to be a particularly challenging benchmark due to the fine-grained visual distinctions among dog breeds.
Compared to datasets such as AFHQ or MNIST, ImageNet-Dogs exhibits significantly higher intra-class variability and subtle inter-class differences, making spectral recovery from partially observed kernels substantially more difficult.
Following the same experimental protocol as in the main text, we construct a density kernel and randomly mask $90\%$ of the off-diagonal entries.
The results in figure~\ref{fig:dog_tsne} indicate that the Max-VNE principle is still able to recover some meaningful spectral structure under severe partial observation, even in highly challenging fine-grained settings.

\begin{figure}
    \centering
    \includegraphics[width=\textwidth]{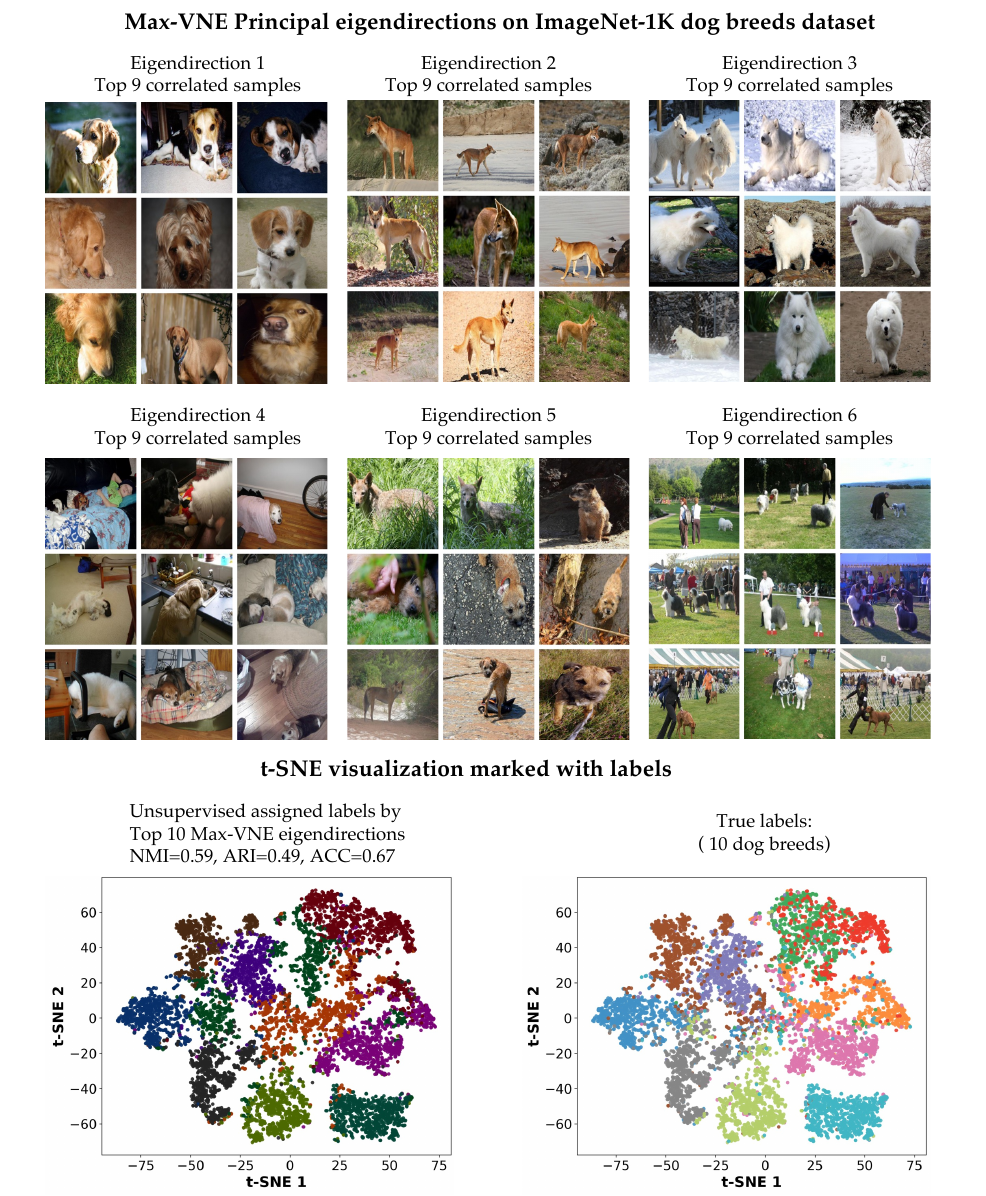}
    \caption{The experimental results of Max-VNE kernel completion from 10\% observed entries on the ImageNet-1k dog breeds dataset. The clustering performance evaluation of the recovered kernel matrix is performed using Normalized Mutual Information (NMI), Adjusted Rand Index (ARI), and Accuracy (ACC).}\vspace{-6mm}

    \label{fig:dog_tsne}
\end{figure}

\paragraph{Ablation on the Factorization Rank $r$}
We study the effect of the factorization rank $r$ used in kernel matrix completion.
Experiments are conducted on the AFHQ dataset with a fixed masking ratio of $90\%$, while varying the rank $r \in \{10, 20, 50\}$.
All other optimization settings are kept identical.
In the figure~\ref{fig:ablation_r}, we observe that the recovery performance and downstream spectral clustering results are already stable for very small ranks.

% This behavior suggests that the underlying kernel admits a low effective rank, with its dominant spectral structure captured by a small number of leading eigen-directions.
% From a maximum entropy perspective, this result is expected: once the principal eigenspaces are determined by the observed constraints, allocating additional degrees of freedom does not increase entropy and may instead fit noise.
% Overall, the ablation demonstrates that the proposed Max-VNE-based kernel completion method is robust to the choice of rank and does not require careful tuning of this parameter.

\begin{figure}
    \centering
    \includegraphics[width=\textwidth]{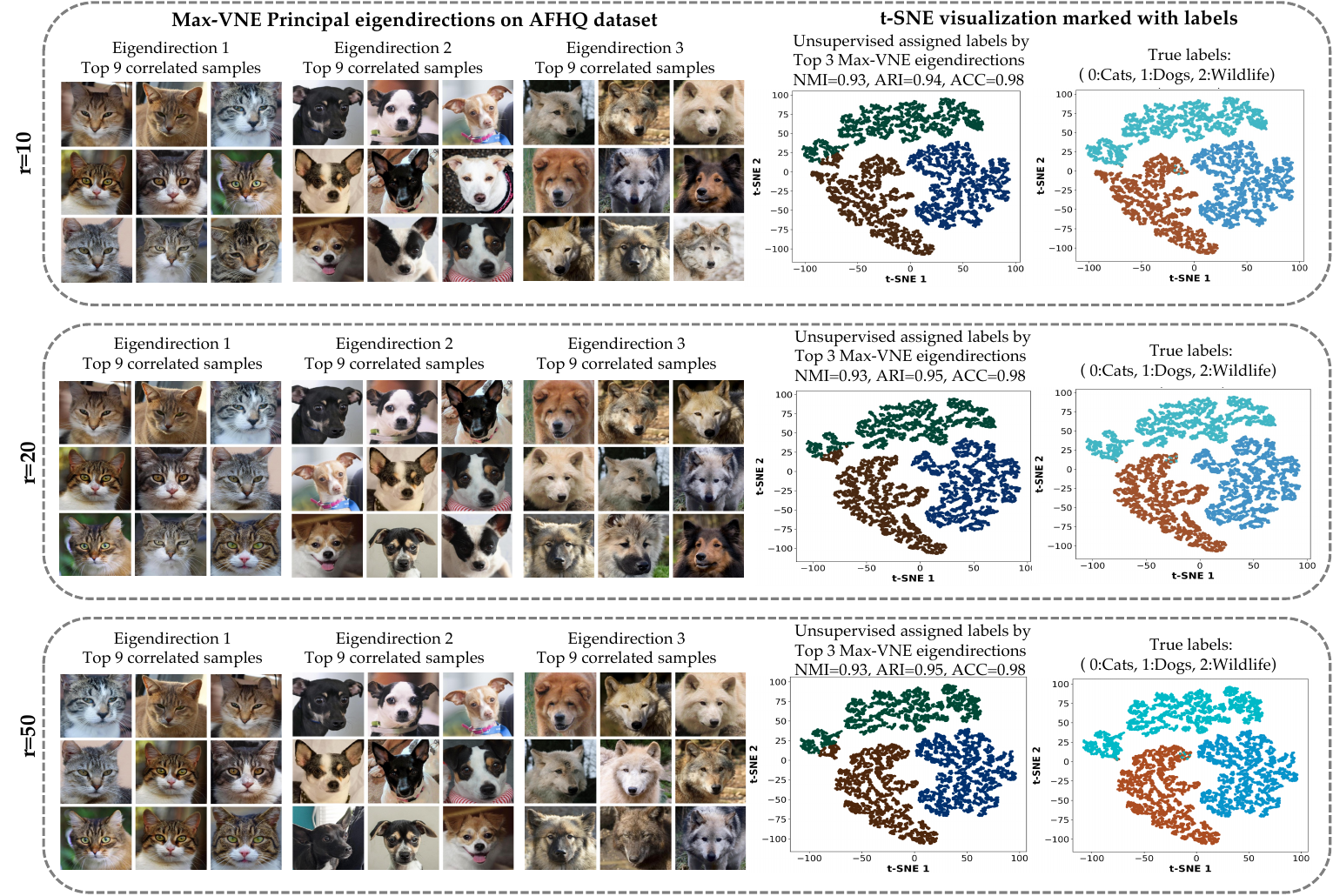}
    \caption{Ablation on the factorization rank $r$ for Max-VNE kernel completion on the AFHQ dataset with $10\%$ observed entries. 
    Clustering performance of the recovered kernel matrix is evaluated using Normalized Mutual Information (NMI), Adjusted Rand Index (ARI), and clustering accuracy (ACC).}
    \label{fig:ablation_r}
\end{figure}

\paragraph{Ablation on the Observation Ratio.}
We study the effect of the observation ratio on the AFHQ dataset with $1\%$, $5\%$, and $10\%$ randomly observed off-diagonal entries, while keeping all other settings fixed.
In figure~\ref{fig:ablation_mask}, we observe that the recovery performance improves when increasing the observation ratio from $1\%$ to $5\%$, and becomes stable thereafter.

\begin{figure}
    \centering
    \includegraphics[width=\textwidth]{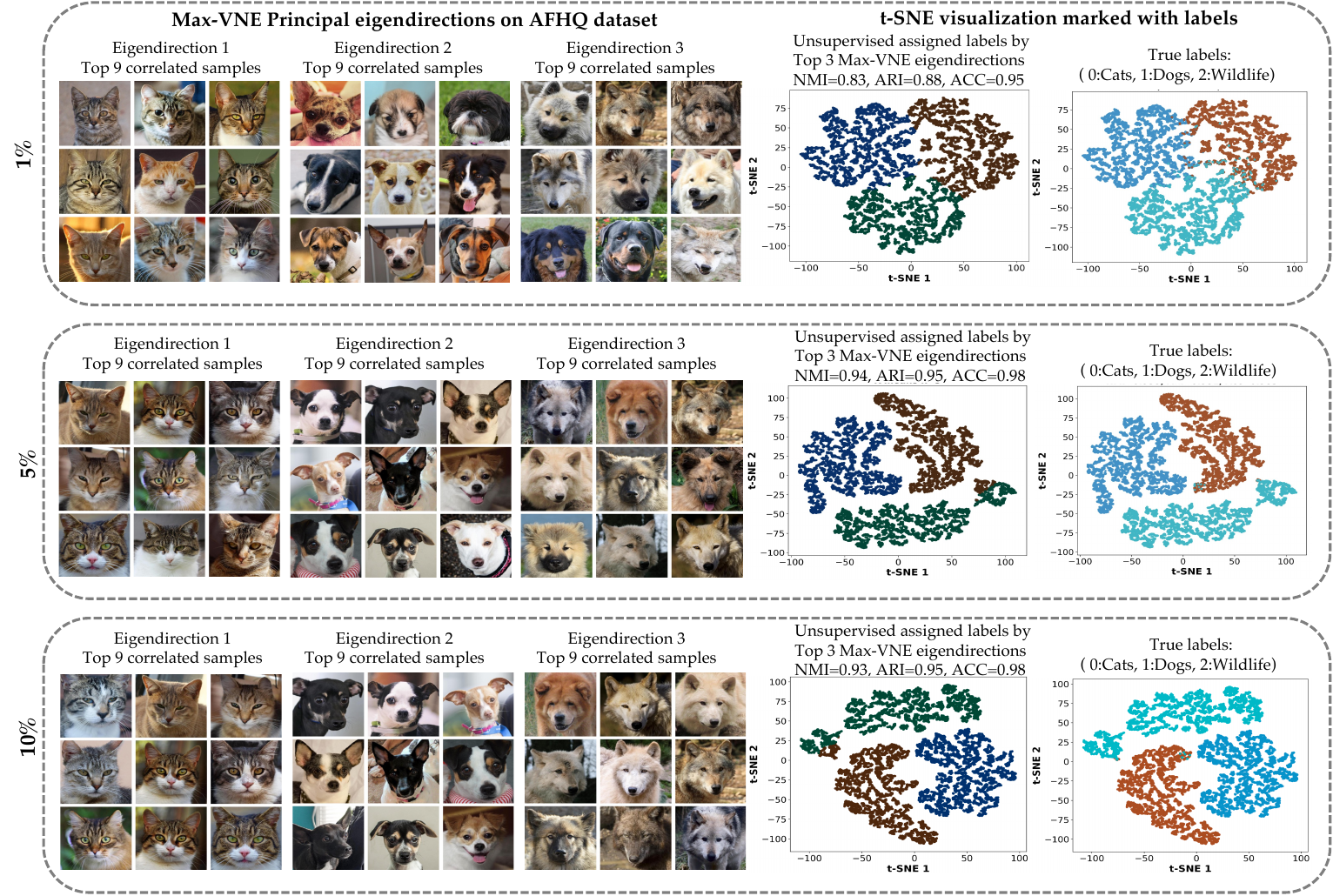}
    \caption{Ablation on the different observation ratio of kernel matrix on the AFHQ dataset with factorization rank fixed to 50. 
    Clustering performance of the recovered kernel matrix is evaluated using Normalized Mutual Information (NMI), Adjusted Rand Index (ARI), and clustering accuracy (ACC).}
    \label{fig:ablation_mask}
\end{figure}

\end{appendices}

\end{document}